%% file: main.tex
\documentclass[nohyperref]{article}
\usepackage{arxiv}
\usepackage{natbib}
\usepackage{makecell}

\usepackage[utf8]{inputenc} 
\usepackage[T1]{fontenc}    
\usepackage{url}            
\usepackage{booktabs}       
\usepackage{nicefrac}       
\usepackage{microtype}      

\usepackage{subfigure}
\usepackage{amsthm,amsfonts,amsmath}
\usepackage[colorlinks=true,linkcolor=blue,citecolor=cyan, pagebackref=true]{hyperref}%
\renewcommand*\backref[1]{\ifx#1\relax \else (Cited on #1) \fi}

\title{Origins of Low-dimensional Adversarial Perturbations}

\usepackage{alias}
\usepackage{mdframed}
\usepackage{amsmath,amsfonts,amssymb}
\usepackage{xcolor}
\usepackage{thm-restate}
\usepackage{url}
\usepackage{booktabs}
\usepackage{tabu}
\usepackage[english]{babel}
\usepackage{amsmath}
\usepackage{graphicx,wrapfig}
\usepackage{comment}
\usepackage{tikz}
\usetikzlibrary{shapes.geometric}
\usetikzlibrary{arrows.meta,arrows}
\usepackage{cancel}
\usepackage{float}
\usepackage{enumitem}

\renewcommand*\backref[1]{\ifx#1\relax \else (Cited on #1) \fi}

\author{%
  \name Elvis Dohmatob$^1$ \email{dohmatob@fb.com}\\
  \name Chuan Guo$^1$
  \email{chuanguo@fb.com}\\
  \name Morgane Goibert$^2$
  \email m.goibert@criteo.com\\
    \addr{$1$ Facebook AI Research}\\
  \addr{$2$ Criteo AI Lab}
}

\usepackage{afterpage}

\begin{document}

\maketitle

\input{abstract}

\setcounter{tocdepth}{2}

\tableofcontents

\input{introduction}

\input{smooth}

\input{local_linear}

\input{compact.tex}
\input{conclusion}


\clearpage
\typeout{} 
\bibliographystyle{icml2022}
\bibliography{bibliography}


\input{appendix.tex}

\end{document}

%% file: abstract.tex
\begin{abstract}

In this paper, we initiate a rigorous study of the phenomenon of low-dimensional adversarial perturbations (LDAPs) in classification. Unlike the classical setting, these perturbations are limited to a subspace of dimension $k$ which is much smaller than the dimension $d$ of the feature space. The case $k=1$ corresponds to so-called universal adversarial perturbations (UAPs; Moosavi-Dezfooli et al., 2017). First, we consider binary classifiers under generic regularity conditions (including ReLU networks) and compute analytical lower-bounds for the fooling rate of any subspace. These bounds explicitly highlight the dependence of the fooling rate on the pointwise margin of the model (i.e., the ratio of the output to its $L_2$ norm of its gradient at a test point), and on the alignment of the given subspace with the gradients of the model w.r.t.inputs. Our results provide a rigorous explanation for the recent success of heuristic methods for efficiently generating low-dimensional adversarial perturbations. Finally, we show that if a decision-region is compact, then it admits a universal adversarial perturbation with $L_2$ norm which is $\sqrt{d}$ times smaller than the typical $L_2$ norm of a data point.
Our theoretical results are confirmed by experiments on both synthetic and real data.
\end{abstract}




%% file: introduction.tex

\section{Introduction}
Despite their widespread use and success in solving real-life tasks like speech recognition, face recognition, assisted driving, etc., neural networks (NNs) are known to be vulnerable to \textit{adversarial perturbations}, i.e. imperceptible modifications ofinput data causing the model to fail \citep{szegedy2013intriguing}.
\input{related_works}

\input{preliminaries}

%% file: related_works.tex
\label{subsec:literature}
Our work is motivated by the empirical observation that adversarial examples are abundant in low-dimensional subspaces as evidenced by several query-efficient black-box attacks. \cite{chen2017zoo} used a finite-difference approximation for the gradient to perform gradient-ascent search. This method inspired others such as Boundary Attack~\citep{brendel2017decision}, NES~\citep{ilyas2018black}, SimBA~\citep{guo2019simple} and HopSkipJump~\citep{chen2020hopskipjumpattack} that approximate the full finite-difference gradient via a Monte-Carlo estimate which sub-samples the coordinates randomly. This approach only requires sampling a very small fraction of the totalinput space, e.g., on ImageNet where theinput dimensionality is approximately $150$K, SimBA perturbs as few as $1 665$ random coordinates and succeeds with over $98.6\%$ probability~\citep{guo2019simple}. Subsequent works also performed adversarial search in a \emph{fixed} subspace such as the low-frequency subspace~\citep{yin2019fourier, guo2018low} or by selecting the subspace in a distribution-dependent manner using an independently-trained NN~\citep{tu2019autozoom, yan2019subspace, huang2019black}.

These empirical findings lead us to hypothesize that adversarial perturbations exist with high probability in low-dimensional subspaces. Our work initiates a rigorous study to understand low-dimensional adversarial perturbations (LDAPs). We provide rigorous explanations for the empirical success of some powerful heuristics that have appeared in the literature \citep{unipert,singularart,guo2018low,yin2019fourier,chen2020hopskipjumpattack}.

\subsection{Main contributions}
\label{subsec:contribs}
Our main results are as follows (details are given in Section \ref{sec:summary}):

\textbf{-- Smooth decision-regions.}
First, we consider different realistic notions of smoothness for a binary classifier. These smoothness assumptions allow us to locally linearize the decision-boundary and derive generic lower-bounds on the fooling rate of any subspace $V$ of the feature space $\mathbb R^d$. The bounds reveal (1) the role of the distribution of the pointwise margin of the classifier; (2) the alignment of $V$ with the unit-normals at the decision boundary. For random subspaces of sufficiently high dimension~\citep{guo2019simple} and subspaces obtained via SVD on the gradients ~\citep{unipert,singularart}, our results provide transparent lower-bounds on the fooling rate, which explain the empirical success of the very efficient heuristic methods that have been proposed in the literature for constructing LDAPs.
Moreover, the lower-bounds only depend on the distributions of the predictions and the gradients of the model and so can be empirically estimated on held-out data, making them a practical predictor for the adversarial vulnerability of classifiers. This is confirmed with experiments on different models and datasets.



\textbf{-- Compact decision-regions.}
Next, we establish the existence of universal adversarial perturbations (UAPs) under compactness assumptions. More precisely,
in the case where a decision-region is an arbitrary compact subset of $\mathbb R^d$, a single perturbation vector of size $\sqrt{d}$ times smaller than the typical $L_2$-norm of a typical data point, is sufficient to achieve a fooling rate close to 100\% on the opposite decision-region.
This result is a consequence of the \emph{Riesz-Sobolev rearrangement inequality} \citep{riez}, which allows us to reduce the situation to the case of a spherical decision-region of the same volume.

Our theoretical results are confirmed by numerous experiments on real and simulated data. In all cases, the bounds can be easily evaluated and are close to the actual fooling rates.

\subsection{Literature overview}
Earlier experiments showed that adversarial attacks based on a single direction of feature space (i.e., UAPS) can be designed to effectively fool neural networks \citep{unipert, singularart}.
UAPs are often more transferable across datasets and architectures than classical attacks, making them interesting for use in practice. Their theoretical analysis has been initiated in \cite{moosavi17}, where the authors establish lower-bounds for the fooling rate of UAPs under certain curvature conditions on the decision boundary. The aforementioned work has two fundamental limitations. First, the notions of curvature used are stated in terms of unconstrained optimal adversarial perturbation (i.e., the closest point) for an arbitraryinput point, and thus are not easy to verify in practice. Also, the existence of the UAP is only guaranteed within a subspace which is required to satisfy a global alignment property with the gradients of the model. In contrast, we use a more flexible curvature requirement (refer to Definition \ref{df:viableV}), which is adapted to any subspace under consideration, and prove results that are strong enough to provide a satisfactory theory of LDAPs, and UAPs in particular, under very general settings.

\cite{guo2020threats} studied LDAPs when the attacker is constrained to a uniformly random $k$-dimensional subspace. For classifiers whose decision-regions are half-spaces and spheres in $\mathbb R^d$, they established the existence of low-dimensional adversarial subspaces under a Gaussian concentration assumption on the data.
Our work considers more general decision-regions (e.g. of certain neural networks) and more general data distributions and subspaces. Our results recover the findings of \cite{guo2020threats} as special cases.

%% file: preliminaries.tex
\section{Preliminaries}
\input{notations}

\subsection{Binary classification and adversarial examples}
\label{subsect:pbsetup}
We consider a binary classification setup, where $X=(X_1,...,X_d) \in \mathbb{R}^d$ denotes aninput of dimension $d$ drawn from an (unknown) probability distribution $\mathbb{P}_X$ (e.g. for the MNIST dataset, $d=784$). We will denote by $f : \mathbb{R}^d \to \mathbb{R}$ the feature map, and $h = \text{sign} \circ f$ the corresponding classifier. 
The binary classifier $h$ can be unambiguously identified with a measurable subset of $\mathbb R^d$
\begin{eqnarray}
C=\{x \in \mathbb R^d \mid h(x) = -1\} = \{x \in \mathbb R^d \mid f(x) \le 0\},
\label{eq:fC}
\end{eqnarray}
called the negative \emph{decision-region} of $h$. Thus, the complement $C' := \mathbb R^d \setminus C$ of $C$ is the positive \emph{decision-region} of $h$. Of course, the terms "negative" or "positive" are interchangeable, as we can always consider the classifier $-h$ instead.
Therefore, without loss of generality, we shall focus our attention on adversarial attacks on the positive decision-region $C'$. 

For example, for NNs, $f(x)$ would be the predicted \emph{logit}, for a closed ball of radius $r>0$ in $\mathbb R_d$, $f(x) := (\|x\|^2-r^2)/2$, and for a half-space (linear classifier), $f(x) := x^\top w - b$.
 


Given aninput $x \in C'$ classified by $h$ as positive, an adversarial perturbation for $x$ is a vector $\Delta x \in \mathbb R^d$ of size $\|\Delta x\|$ such that $x + \Delta x \in C$. 
The goal of the attacker is to move points from $C'$ to $C$ with small perturbations.
Note that we are not interested in the true labels of theinputs, just the robustness of the classifier w.r.t. its own predictions. However, note that this distinction is not important for classifiers which are already very accurate in the classical sense.

\subsection{Low-dimensional adversarial perturbations (LDAPs)}
In this paper, we focus on \textit{low-dimensional} perturbations (LDAPs)~\citep{guo2018low,guo2019simple,tu2019autozoom, yan2019subspace, huang2019black,guo2020threats}, meaning that the perturbations $\Delta x$ are limited to a $k$-dimensional subspace $V$ of $\mathbb R^d$ whose choice is left to the attacker. The special case where $k=1$ corresponds to the scenario where the attacker is allowed to operate in one dimension only (e.g. modify the same pixel in all images of the same class), also famously known as \emph{universal adversarial perturbations (UAPs)} \citep{unipert,singularart}. 
More generally,  given a subspace $V$ of $\mathbb R^d$, let $C^\varepsilon_V$ be the set of all points in $\mathbb R^d$ which can be pushed into the negative decision-region $C$ by adding a perturbation of size $\varepsilon$ in $V$, that is
\begin{eqnarray}
C^\varepsilon_V := C + \varepsilon B_V = \{x \in \mathbb R^d \mid x + v \in C,\text{ for some }v \in V\text{ with }\|v\| \le \varepsilon \},
\end{eqnarray}
where $B_V := V \cap B_d$ is the unit-ball in $V$.
Note that by definition, $x \in C_V^\varepsilon$ iff $(x + \varepsilon B_V) \cap C \ne \emptyset$. In the particular case of full-dimensional attacks where $V = \mathbb R^d$, the set $C^\varepsilon_V$ corresponds to the usual $\varepsilon$-expansion $C^\varepsilon$ of $C$, i.e., the set of points in $\mathbb R^d$ which are at a distance at most $\varepsilon$ from $C$. This case has been extensively studied in \cite{shafahi2018adversarial,fawzi18,mahloujifar2019curse,dohmatob19}. Note that it always holds that $C \subseteq C_V^\varepsilon \subseteq C^\varepsilon$. 


\begin{restatable}[Fooling rate of a subspace]{df}{}
Given an attack budget $\varepsilon \ge 0$, the fooling rate $\fr(V;\varepsilon)$ of a subspace $V \subseteq \mathbb R^d$ is the proportion of test data which can be moved from the positive decision-region $C'$ to the negative decision-region $C$ by moving a distance $\varepsilon$ along $V$, that is
\begin{eqnarray}
{\rm FR}(V;\varepsilon) := \mathbb  P_X(X \in C^\varepsilon_V \mid X \in C').
\end{eqnarray}
\end{restatable}
Note that by definition of $C_V^\varepsilon$, the fooling rate $\fr(V;\varepsilon)$ is a supremum over all possible attackers operating in the subspace $V$, and with $L_2$-norm budget $\varepsilon$. In particular,
${\rm FR}(\mathbb R^d;\varepsilon)$ is the usual optimal fooling rate of an adversarial attack with budget $\varepsilon$, without any subspace constraint. 

\section{Summary of main results and empirical verification}
\label{sec:summary}


\paragraph{High-level overview.}
We first formalize the notion of an \emph{adversarially viable} subspace which is a subspace $V$ that has a non-negligible inner product with the classifier's gradient, hence it is possible to significantly alter the value of $f(x)$ by moving strictly within $V$. Intriguingly, such subspaces are pivotal to the empirical success of LDAPs, and we show that popular heuristics 
lead to adversarially viable subspaces. Then, we prove that when the classifier satisfies certain smoothness conditions, adversarially viable subspaces allow the attacker to follow the gradient direction within $V$ to reach the decision boundary of $C$ for most points $x \in C'$, hence achieving a high fooling rate. Finally, if 
$C'$ is compact, we can obtain a stronger result that UAPs with high fooling rates also exist.

\subsection{Adversarially viable subspaces}
\begin{wrapfigure}{R}{0.35\textwidth}
\centering
\caption{\label{fig:geom_intuition}Adversarial viability.
}
\includegraphics[width=0.35\textwidth]{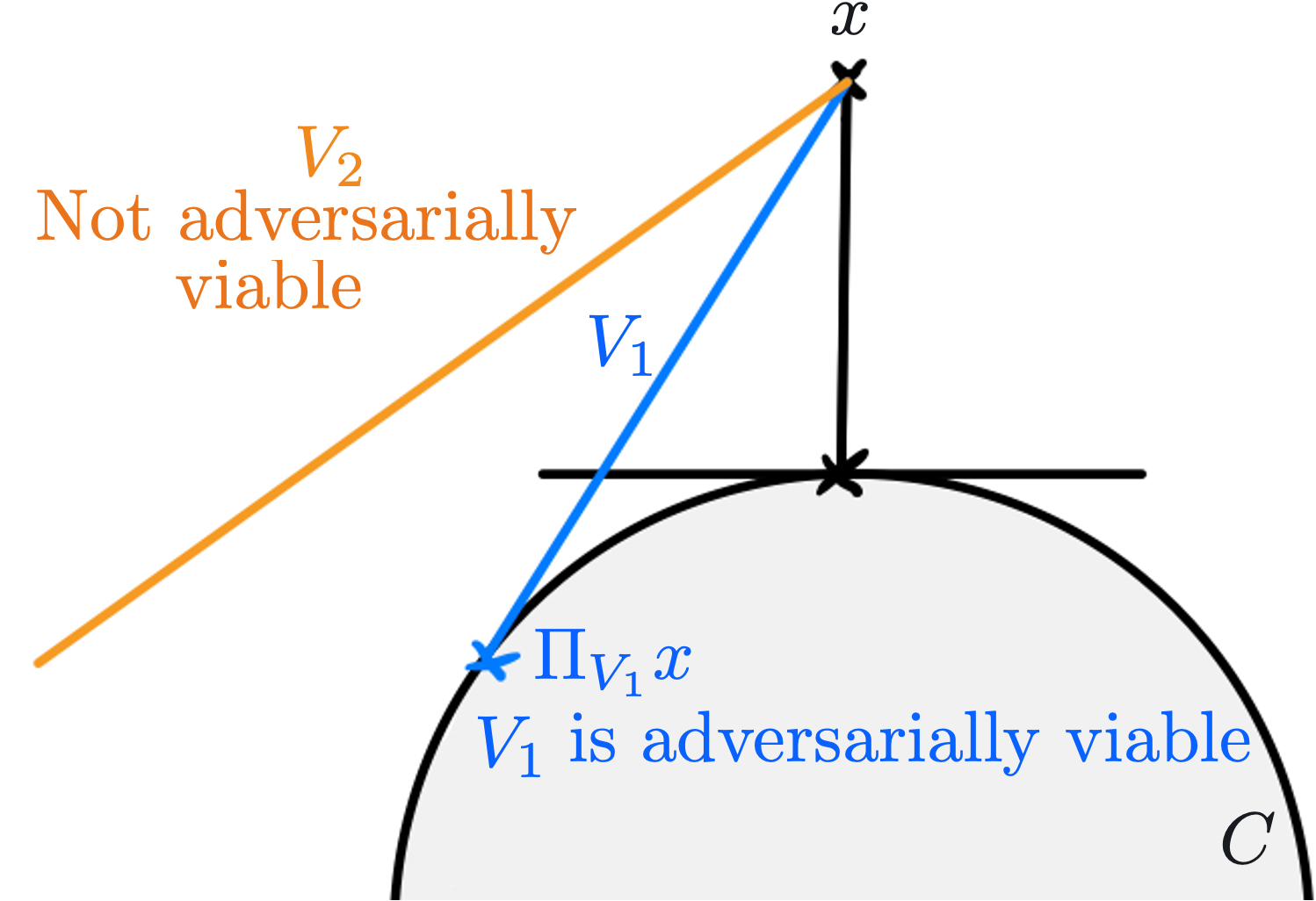} \vspace{-1cm}
\end{wrapfigure}
Restricting the adversarial perturbation to a given subspace $V$ presents a particular challenge to the attacker. 
If $\dim V < d$ and $x \in C':=\mathbb R^d \setminus C \ne \emptyset$, it is possible that $x \not \in C_V^\varepsilon$ for all $\varepsilon>0$. In particular, if a subspace $V$ is orthogonal to the gradient of $f$ at a point $x \in \mathbb R^d$, then no amount of perturbation within $V$ will make $x$ closer to the boundary of $C$. See Fig. \ref{fig:geom_intuition} for underlying geometric intuition. Thus, we can hope to establish nontrivial fooling rates only for certain subspaces.

Our first contribution is a crisp characterization of subspaces for which we can hope to achieve a nonzero fooling rate. These are so-called \emph{adversarially viable} subspaces and are a generalization of the subspaces considered in \cite{moosavi17,unipert} and \cite{guo2020threats}.
\begin{restatable}[Adversarially viable subspace]{df}{}
Given $\alpha \in (0,1]$ and $\delta \in [0,1)$, a possibly random subspace $V \subseteq \mathbb R^d$ is said to be adversarially $(\alpha,\delta)$-viable (w.r.t $C'$) if
\begin{eqnarray}
\mathbb P_{X,V}(\|\Pi_V \eta(X)\| \ge \alpha \mid X \in C')
\ge 1-\delta,
\end{eqnarray}
where $\eta(x):=\nabla f(x)/\|\nabla f(x)\|$ is the gradient direction at $x$.
\label{df:viableV}
\end{restatable}


\paragraph{Examples.}
The now provide some important examples of adversarially viable subspaces.
\begin{restatable}[Random subspaces]{lm}{randomviable}
\label{lm:randomviable}
A uniformly-random $k$-dimensional subspace $V$ of $\mathbb R^d$ is $(\sqrt{k/d}-t,2e^{-t^2d/2})$-viable for any $t \in (0,\sqrt{k/d})$.
\end{restatable}
Such subspaces have been proposed in
\cite{unipert,guo2020threats} as a black-box technique for generating LDAPs

Let $\Sigma_\eta \in \mathbb R^{d \times d}$ be the covariance matrix of the gradient direction $\eta(X)$ conditioned on $X \in C'$.
\begin{restatable}[Gradient eigen-subspaces]{thm}{svd}
\label{thm:svd}
For any $k \in [d]$, let $s_k \in (0,1]$ be the sum of first the $k$ eigenvalues of $\Sigma_\eta$. Then,  for any $\alpha \in (0,\sqrt{s_k})$, the (deterministic) subspace $V_{\mathrm{eigen},k}$ of $\mathbb R^d$ corresponding to the $k$ largest eigenvalues of $\Sigma_\eta$ is adversarially $(\alpha,(1-s_k)/(1-\alpha^2))$-viable.
\end{restatable}
Thus, if the histogram of eigenvalues of $\Sigma_\eta$ is "spiked" in the sense that $s_k \ge s = \Omega(1)$ for some $k=o(d)$, then $V_{\mathrm{eigen},k}$ is a $o(d)$-dimensional adversarially $(\Omega(1),O(1-s))$-viable subspace! This provides a rigorous justification for the heuristic in \cite{unipert,singularart} which proposed UAPs based on eigenvectors of the covariance matrix $\Sigma_\eta$. Our experiments in Section \ref{subsec:experiments} also support this.

\subsection{Smooth decision-regions}
\label{subsec:smoothsummary}
We consider the case where $f$ is differentiable almost everywhere and satisfies
\begin{restatable}[Strong gradients]{cond}{}
\label{cond:beta}
$\mathbb P_X(\|\nabla f(X)\| \ge \beta \mid X \in C')\ge 1-\gamma$, for some $\beta>0$, $\gamma \in [0,1)$.
\end{restatable}
This condition
ensures that there is a strong descent direction at a constant fraction of points in the positive decision-region $C'$, to allow for gradient-based attacks. 

We combine the above condition with different notions of smoothness of the gradient $\nabla f$, namely:

-- \emph{Lipschitzness} (Section \ref{sec:smooth}), where the gradient map $x \mapsto \nabla f(x)$ is assumed to be Lipschitz continuous on the positive decision-region $C':=\mathbb R^d\setminus C$, i.e. $\|\nabla f(x')-\nabla f(x)\| \le L\|x'-x\|$ for all $x,x' \in C'$ and some $L \in [0,\infty)$.
    For example, this is the case when $f$ is a hyper-ellipsoid, an affine function, a feed-forward neural network with bounded weights and a twice-differentiable activation function with bounded Hessian (e.g. sigmoid, quadratic, tanh, GELU~\citep{hendrycks2016gaussian}, cos, sin), etc.
    
    
-- \emph{Local-flatness} (Section \ref{sec:local}), where $\nabla f(x)$ is assumed to be nearly constant on a neighborhood of radius $R \in (0,\infty]$ around each point $x \in C'$. This is the case of ReLU neural networks at initialization and includes neural nets in the random features regime where only the output layer is trained \cite{daniely2020most,bubecksinglestep2021,bartlett2021}, and also neural networks in the so-called \emph{lazy} regime~\cite{RuiqiGao2019}. We also empirically observe that this is the case of fully-trained feedforward neural networks with ReLU activations.

\textbf{Results.}
Let $m_f:\mathbb R^d \to \mathbb R_+$ be the \emph{margin} of the classifier \eqref{eq:fC}, defined by
\begin{eqnarray}
m_f(x) := \max(f(x),0)/\|\nabla f(x)\| = \begin{cases}
f(x)/\|\nabla f(x)\|,&\mbox{ if }x \in C',\\
0,&\mbox{ if }x \in C.
\end{cases}
\label{eq:margin}
\end{eqnarray}
Given a possibly random adversarially $(\alpha,\delta)$-viable subspace $V$ of $\mathbb R^d$, we establish in Theorem \ref{thm:smoothboundary} and Theorem \ref{thm:almostconstantgrads} lower-bounds on the fooling rate of the form
    \begin{equation}
        \mathbb E_V[\fr(V;\varepsilon)] \gtrsim \mathbb P(m_f(X) \le \overline\alpha\varepsilon \mid X \in C')-\delta-\gamma.
        \label{eq:generalbound}
    \end{equation}

Here,
$\overline\alpha \in (0, 1]$ depends on $\alpha$, $\beta$ and the smoothness of $f$ (i.e., on $L$, etc.). Importantly, the generic bound \eqref{eq:generalbound} explicitly highlights the dependence of the fooling rate on the pointwise margin of the classifier and on the alignment of the given subspace with the gradients of the model w.r.t.inputs.


\textbf{Some consequences.}
The $L_2$-norm $R$ of a typical data point is of order $\sqrt{d}$,
while the margin $m_f(X)$ is typically of order $\mathcal O(1)$; this is formally proved in \cite{daniely2020most,bartlett2021} in the case of networks at initialization, and empirically observed in \cite{bengiomargin} for fully-trained networks. Also, as observed in \cite{unipert}, the singular-values of the gradient covariance matrix $\Sigma_\eta$ are typically long-tailed.
Thus, combining with Theorem \ref{thm:svd}, our results predict that for sufficiently large $k \ll d$, the subspace spanned by the top $k$ singular-vectors of $\Sigma_\eta$ has a nonzero fooling rate with attack budget $\varepsilon \asymp 1/\alpha=O(1)$ which is $\sqrt{d}/\varepsilon \gtrsim \sqrt{d}$ times smaller than $R$, the $L_2$-norm of a typical data point. On the other hand, in view of Lemma \ref{lm:randomviable}, for $k \asymp d^{1/2+o(1)}$, a random $k$-dimensional subspace of $\mathbb R^d$ has nonzero fooling rate with $\varepsilon \asymp 1/\alpha = d^{1/4+o(1)}$, i.e. $\sqrt{d}/\varepsilon \asymp d^{1/4-o(1)}$ times smaller than $R$.


\subsection{Compact decision-regions}
\label{subsec:compactsummary}
We consider the scenario where the positive decision-region $C' := \mathbb R^d \setminus C$ of the classifier is a compact subset of $\mathbb R^d$, equipped with the \emph{uniform / volume} measure. In this case, we establish in Theorems \ref{thm:ballcase} and \ref{thm:steiner}, the existence of universal adversarial perturbations (UAPs) that have a fooling rate close to $100\%$, with an attack budget $\varepsilon$ which is $\sqrt{d}$ times smaller than the typical $L_2$-norm of a data point. Moreover, these UAPs can be selected completely at random, without any information from the classifier.
Our proof uses the classical \emph{Riesz-Sobolev rearrangement inequality} \citep{riez}, which allows us to reduce the situation to the case where $C'$ is a ball of the same volume.
The lower-bounds we obtain are then a direct consequence of the \emph{curse of dimensionality}. 

We conjecture that the uniformity assumption on the distribution of the data in the positive-decision region can probably be replaced by assuming that the distribution of $X$ conditioned on $X \in C'$ has density which is bounded-away from zero. This extension is left for future work.


\subsection{Empirical verification}
\label{subsec:experiments}
\paragraph{Smooth decision-regions.}
Our results are empirically verified in Fig. \ref{fig:relu} (random subspace attacks) and Fig. \ref{fig:mnist-lenet} (singular subspace attack).
\begin{figure}[!h]
    \centering
    \begin{subfigure}{}
    \includegraphics[height=0.044\textwidth]{{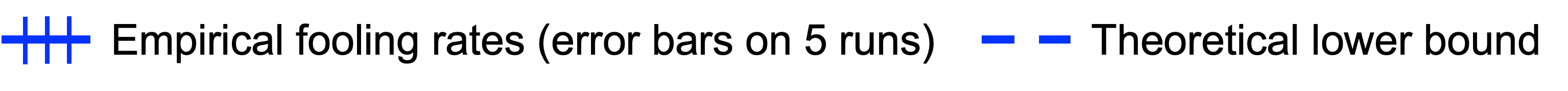}}
    \end{subfigure}
    \vspace{-.1cm}
    \centering
    \begin{subfigure}{Two-layer ReLU NN at initialization:input dim. $d = 784$, width $d_1=100$. Simulated data.}
    \includegraphics[width=1.0\textwidth]{{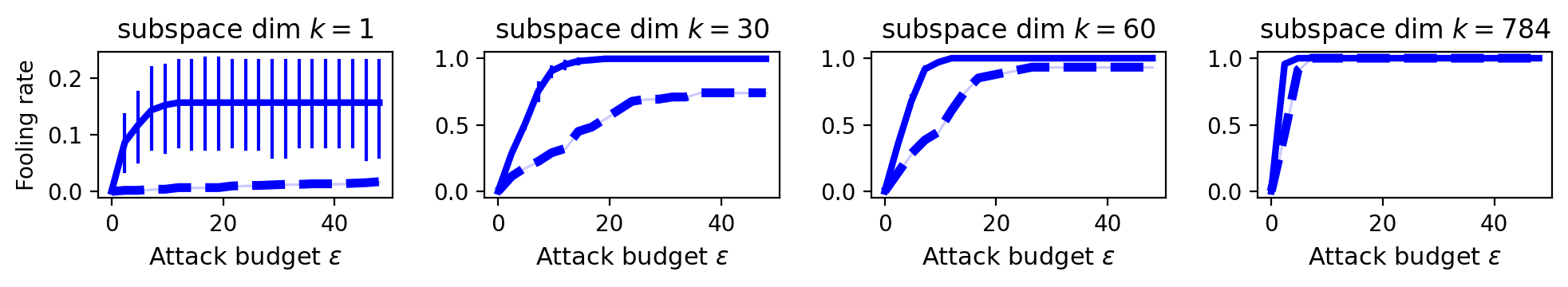}}
    \end{subfigure}
    \vspace{-.1cm}
     \begin{subfigure}{Two-layer ReLU NN in RF regime:input dim $d = 784$, width $d_1=100$. Simulated data.}
    \includegraphics[width=1\textwidth]{{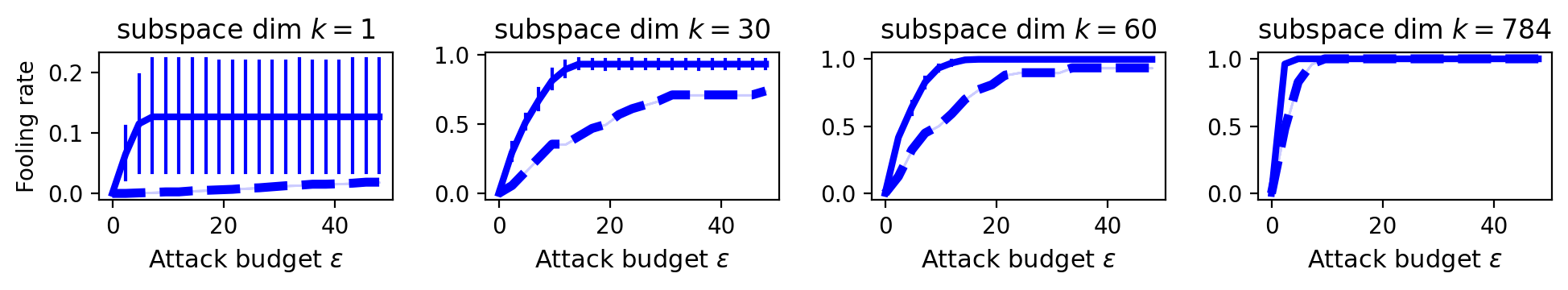}}
    \end{subfigure}
    \vspace{-.1cm}
     \begin{subfigure}{Full-trained LeNet (conv layers + dense layers + ReLU activation) on MNIST dataset.}
    \includegraphics[width=1\linewidth]{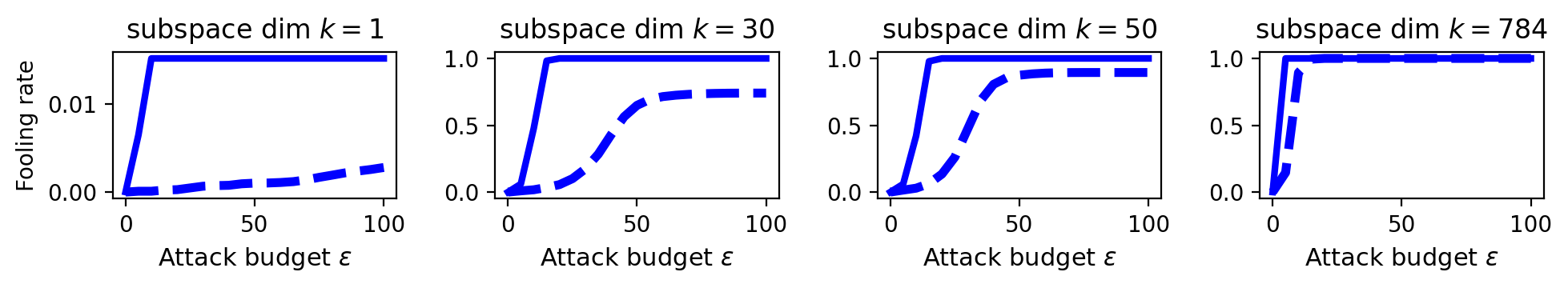}
    \end{subfigure}
    \vspace{-.5cm}
    \caption{(Random subspace attack) Empirical confirmation of our results. Broken lines correspond to our theoretical lower bound. Solid curves correspond to empirically computed fooling rates, with error-bars accounting for randomness in the initialization of the network, over $5$ independent runs.
    }
    \label{fig:relu}
\end{figure}
In Fig. \ref{fig:relu} (first and second row), the distribution $P_X$ of the features is $N(0,I_d)$, and the training labels are given from a simple linear model:  $y_i = x_{ij}$. For MNIST data~\citep{lecun-mnisthandwrittendigit-2010} (third row), we construct a binary classification problem by restricting it to the digits 0 and 8. As in \cite{guo2020threats}, we run FGSM attacks on a randomly chosen subspace $V$ (of different dimensions) of the feature space $\mathbb R^d$, and report the fooling rates (solid lines) and compare them with our proposed lower-bounds \eqref{eq:generalbound}. As we can see from the figure, in all the cases, the lower-bounds (broken lines) are verified.

In Fig. \ref{fig:mnist-lenet}, we consider the same experimental setting in Fig. \ref{fig:relu}.  We use $n=1000$ random examples $x_1,\ldots,x_n$, and compute the empirical covariance matrix $\widehat{\Sigma}_\eta := (n-1)^{-1}\sum_{i=1}^n(\eta_i-\overline{\eta})(\eta_i-\overline{\eta})^\top$, of the gradient directions $\eta_i := \eta(x_i)$, with $\overline{\eta} := (1/n)\sum_{i=1}^n\eta_i$. As in \cite{singularart}, we extract the top eigenvector of $\widehat\Sigma_\eta$ and use it as a universal perturbation vector for a separate test set. In the leftmost subplot, we show a histogram of eigenvalues. Notice how the largest eigenvalue for each model is much larger than the other eigenvalues. Thanks to Theorem \ref{thm:svd}, this means that the principal eigenvector $v$ spans an adversarially viable subspace. This is confirmed in the 2nd, 3rd, and 4th subplots where we see that fooling rate rises rapidly as a function of the attack budget $\varepsilon$. We see from the figure that our predicted lower-bounds (broken) lines are satisfied in all cases.

Full details of the experimental setup and code for reproducing the results are provided as part of the supplementary materials.

\begin{figure}[!ht]
    \centering
    \begin{subfigure}{}
    \includegraphics[height=0.044\textwidth]{{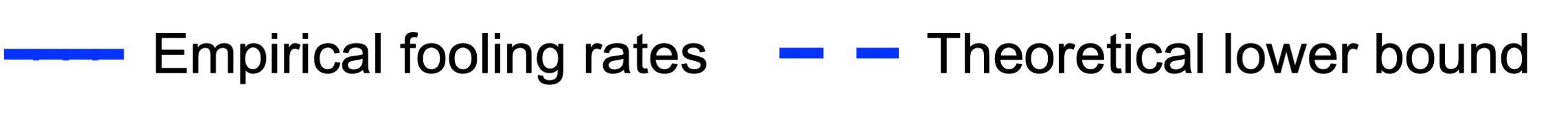}}
    \end{subfigure}~\\
    \vspace{-.1cm}
    \centering
    \begin{subfigure}{}
    \includegraphics[width=.27\linewidth]{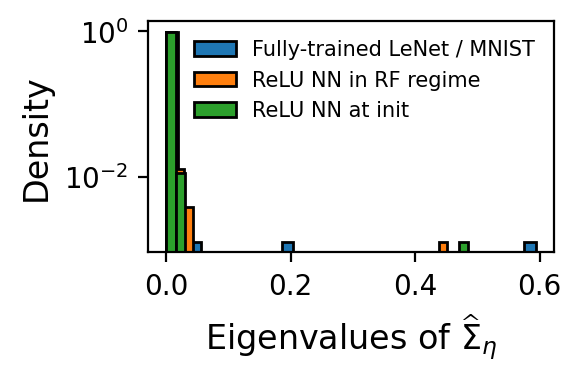}\hspace{-.1cm}
    \includegraphics[width=.72\linewidth]{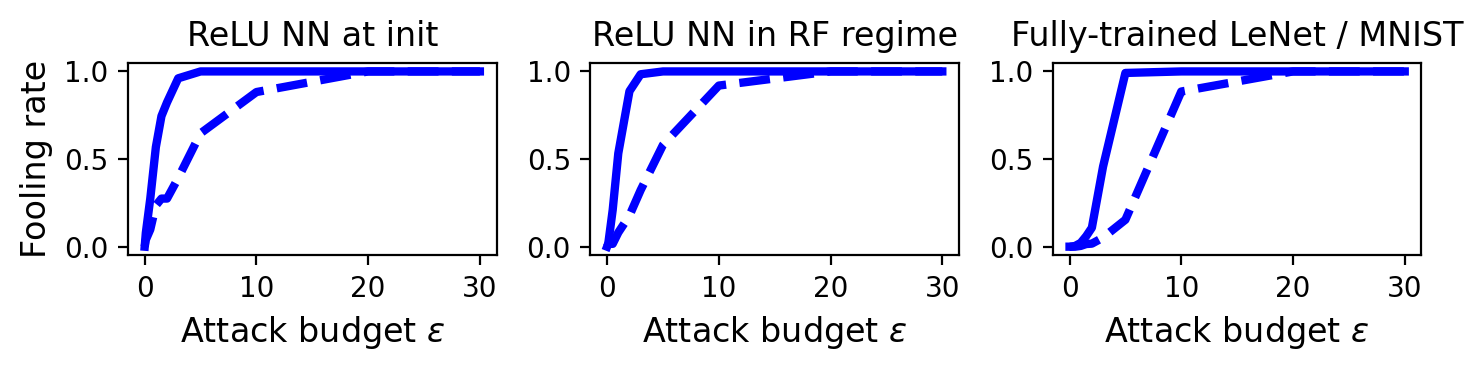}
    \end{subfigure}
    \vspace{-.5cm}
\caption{(Singular-subspace attack). Same experimental setting in Fig. \ref{fig:relu}.  \textbf{Leftmost plot:} Showing a histogram of the eigenvalues of empirical covariance matrix $\widehat{\Sigma}_\eta$ of gradient directions (computed on 1000 examples). Notice how the largest eigenvalue for each model is much larger than the other eigenvalues.
\textbf{Second to fourth (rightmost) plot:}
Notice how the fooling rate rises rapidly.}
    \label{fig:mnist-lenet}
\end{figure}

\paragraph{Compact decision-regions.}
We run a small toy experiment where the positive decision-region $C':=\mathbb R^d \setminus C$ is the unit-ball $B_d$ in $\mathbb R^d$. The attack subspace is the span of any unit-vector in $\mathbb R^d$ chosen at random.  We used a batch of $n=10^4$ points sampled uniformly at random on $C'= B_d$. Solid curves correspond to actual fooling rates computed on a batch of $n=10^4$ points sampled uniformly at random from $B_d$. The results are shown in Fig. \ref{fig:ballcase}.

\begin{figure}[!htp]
    \centering
    \begin{subfigure}{}
    \includegraphics[height=0.044\textwidth]{{legend_low_dim2.png}}
    \end{subfigure}
    \vspace{-.1cm}
    \centering
    \begin{subfigure}{}
    \includegraphics[width=1\linewidth]{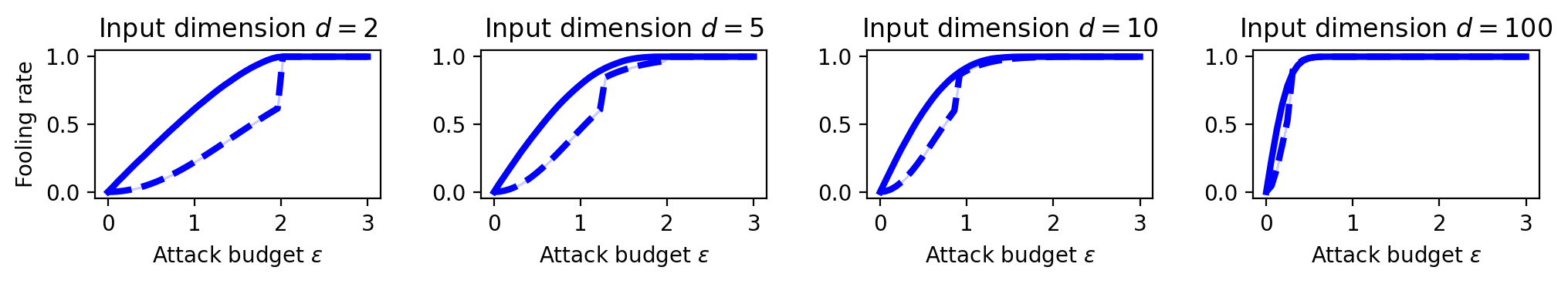}
    \end{subfigure}
    \vspace{-.5cm}
    \caption{Universal adversarial attacks for compact decision-region. Solid curves correspond to actual fooling rates while broken lines correspond to the lower-bound predicted by Theorem \ref{thm:ballcase}.
    }
    \label{fig:ballcase}
\end{figure}

%% file: notations.tex
\paragraph{Notations.} $[N]$ denotes the integers from $1$ to $N$ inclusive, $t_+$ the maximum of $t$ and $0$, $\|u\|$ the $L_2$-norm (unless otherwise stated) of a vector $u$, and $\|A\|_{op}$ denotes the operator norm of a real matrix $A$. The unit-sphere (resp. closed unit-ball) in $\mathbb R^d$ is written $\sphere$ (resp. $B_d$). The orthogonal projection of a vector $z \in \mathbb R^d$ onto the subspace $V \subseteq \mathbb R^d$ is denoted $\Pi_V z$.
As usual, asymptotic notation $F(d) = O(G(d))$ (also written $F(d) \lesssim G(d)$) means there exists a constant $c$ such that $F(d) \le c\cdot G(d)$ for sufficiently large $d$, while $F(d) = \Omega(G(d))$ means $G(d) = O(F(d))$, and $F(d) = \Theta(G(d))$ or $F(d) \asymp G(d)$ means $F(d) \lesssim G(d) \lesssim F(d)$. Finally, $F(d) = o(G(d))$ means $F(d)/G(d) \to 0$ as $d \to \infty$.

%% file: smooth.tex
\section{Results for Lipschitz smooth decision-boundaries}
\label{sec:smooth}
Consider a binary classifier on $\mathbb R^d$ for which the negative decision-region $C$ of the classifier \eqref{eq:fC},
where $f:\mathbb R^d \to \mathbb R$ is a differentiable function.

\subsection{Warmup: All you need are strong gradients}
Thanks to a classical result from optimization theory (see Proposition 3.2 of \cite{nonlinearhoffman}), if the condition
\begin{restatable}[Uniformly strong gradients]{cond}{}
\label{cond:strongslope}
$\|\nabla f(x)\| \ge \beta$ for all $x \in C'$.
\end{restatable}
is satisfied,
then any $x \in C'$ is at a distance $d_C(x)$ at most $f(x)/\beta$ from $C$.
Note that Condition \ref{cond:strongslope} is a special case of Condition \ref{cond:beta} corresponding to $\gamma=1$. Intuitively, under Condition \ref{cond:strongslope}, the gradient of $f$ at any point $x \in C'$ is strong enough: gradient-flow started at $x$ then escapes the region $C'$ after travelling a distance $O(f(x))$.
This is formalized in the following result which will be extended to the case of subspace attacks in the rest of this section. 
\begin{restatable}[A lower-bound for full-dimensional attacks]{thm}{}
\label{thm:strongslope}
Under Condition \ref{cond:strongslope}, it holds that
\begin{eqnarray}
\fr(\mathbb R^d;\varepsilon) := \mathbb P_X(X \in C^\varepsilon \mid X \in C') \ge \mathbb P_X(f(X) \le \beta\varepsilon \mid X \in C'),\text{ for any }\varepsilon \ge 0.
\end{eqnarray}
\end{restatable}
 As an illustration, if we consider $f$ to be a randomly initialized\footnote{With layer widths within $\mathrm{poly}(\log d)$ factors of one another, and weights initialized in the standard way.} finite-depth ReLU neural-network, one can show (see~\cite{daniely2020most,bubecksinglestep2021, bartlett2021}) that for any $x \in \mathbb R^d$, we have $f(x) = \mathcal O(\|x\|/\sqrt{d})$ and $\inf_x \|\nabla f(x)\| = \Omega(1)$ w.h.p. over the weights. The above theorem immediately predicts the existence of adversarial examples of size $\sqrt{d}$ times smaller than the typical $L_2$-norm of data point.


\subsection{Main result under Lipschitzness}
We will extend Theorem \ref{thm:strongslope} to the case of subspace attacks, under the following smoothness condition
\begin{restatable}[Lipschitz gradients]{cond}{}
There exists $L \in [0,\infty)$
such that
\begin{align}
\|\nabla f(x') - \nabla f(x)\| \le L\|x'-x\|\,\forall x,x' \in C'.
\label{eq:lipgrad}
\end{align}
\label{cond:smooth}
\end{restatable}
\vspace{-.5cm}
This condition stipulates that the gradient of $f$ varies smoothly on the positive decision-region $C' = \mathbb{R}^d \setminus C$ of the classifier \eqref{eq:fC}. Note that when $f$ is twice-differentiable on $C'$, Condition \ref{cond:smooth} holds with
$ L = \sup_{x \in C'}\|\nabla^2 f(x)\|_{op},
$
where $\nabla^2 f(x) \in \mathbb R^{d \times d}$ is the Hessian of $f$ at $x$. For example, a feed-forward neural net with bounded weights and twice-differentiable activation function with bounded Hessian (e.g. sigmoid, quadratic, tanh, GELU, cos, sin, etc.) will satisfy Condition \ref{cond:smooth}.




The following is one of our main results. It generalizes both Proposition \ref{prop:halfspace} and Theorem \ref{thm:strongslope}.
\begin{restatable}[Subspace attacks for smooth decision-boundaries]{thm}{smoothboundary}
\label{thm:smoothboundary}
Suppose Condition \ref{cond:smooth} is in order and let $V$ be a possibly random adversarially $(\alpha,\delta)$-viable subspace of $\mathbb R^d$.
Then,
\begin{itemize}
\item[(A)]
For any $\varepsilon \ge 0$, the average fooling rate of $V$ is lower-bounded as follows
\begin{eqnarray}
\label{eq:smoothlowerbound}
\mathbb E_V [\fr(V;\varepsilon)] \ge \mathbb P_X\left(m_f(X) \le \min\left(\frac{\alpha\varepsilon}{2},\frac{\alpha^2 \| \nabla f(X)\|}{2L}\right) \,\Big | \, X \in C'\right)-\delta.
\end{eqnarray}

\item[(B)] If in addition Condition \ref{cond:beta} is in order, then for any $0 \le \varepsilon \le \alpha\beta/L$ it holds that
\begin{eqnarray}
\label{eq:smoothlowerabsolutebound}
\mathbb E_V [\fr(V;\varepsilon)] \ge \mathbb P_X\left(m_f(X) \le \alpha\varepsilon/2 \mid X \in C'\right)-\delta-\gamma.
\end{eqnarray}
\end{itemize}
\end{restatable}
\begin{restatable}{rmk}{}
Note that the condition "$\varepsilon \le \alpha\beta/L$" in part (B) of the theorem cannot be removed in general, as is seen in the case where $C=B_d$, and considering any subspace $V$ with $\dim V < d$.
\end{restatable}

In Appendix \ref{subsec:warmup} and \ref{subsec:lipschitzexamples}, bounds established in \cite{guo2020threats} are recovered from Theorem \ref{thm:smoothboundary} as special cases.

\subsection{A matching upper-bound under convexity}
We now establish a corresponding upper bound for the case where $C$ is convex (e.g., half-spaces, balls, ellipsoids, etc.). See Appendix \ref{subsec:upperboundproof} for proof.
\begin{restatable}{thm}{upperbound}
\label{thm:upperbound}
Suppose $f$ is convex differentiable, and let $V$ be a subspace of $\mathbb R^d$ satisfying
\begin{eqnarray}
\|\Pi_V \eta(x)\| \le \widetilde\alpha,\text{ for some }\widetilde\alpha \in [0,1]\text{ and for all }x \in C'.
\end{eqnarray}
Then, for any $\varepsilon \ge 0$, we have
$\fr(V;\varepsilon) \le \mathbb P_X(m_f(X) \le \widetilde{\alpha}\varepsilon \mid X \in C')$. In particular, for any subspace $V$ of $\mathbb R^d$, it always holds that $\fr(V;\varepsilon) \le \mathbb P_X(m_f(X) \le \varepsilon \mid X \in C')$.

\end{restatable}

%% file: local_linear.tex
\section{Results for locally almost-affine decision-regions}
\label{sec:local}
We now consider the following smoothness condition for the classifier \eqref{eq:fC}.
\begin{restatable}[Bounded oscillation of gradients]{cond}{}
The exists $0 < R \le \infty$ and $0 \le \theta \ll 1$ such that
\begin{eqnarray}
\|\nabla f(x')-\nabla f(x)\| \le \theta\text{ for all }x,x' \in C'\text{ with }\|x'-x\| \le R.
\end{eqnarray}
\label{cond:almostconstgrad}
\end{restatable}
\vspace{-.65cm}
Examples of functions that satisfy this condition include: half-spaces and wide feedforward ReLU neural networks with randomly initialized intermediate weights, where $\theta = o(1)$ w.h.p. over the intermediate weights, as will be seen in Section \ref{subsec:relu}.
The following is one of our main contributions.
\begin{restatable}{thm}{almostconstantgrads}
\label{thm:almostconstantgrads}
Suppose Conditions \ref{cond:beta} and \ref{cond:almostconstgrad} with parameters $\beta \in (0,\infty)$, $R \in (0,\infty]$ and $0 \le \theta \ll 1$. Let $V$ be a possibly random adversarially $(\alpha,\delta)$-viable subspace of $\mathbb R^d$ with $\alpha > \theta/\beta$.
Then, for any $0 \le \varepsilon \le R$, the average fooling rate of $V$ is lower-bounded as follows (with $\overline \alpha := 1-\theta/(\alpha\beta)$)
\begin{eqnarray}
\label{eq:notfullV}
\mathbb E_V[\fr(V;\varepsilon)]  \ge \mathbb P_X(m_f(X) \le \overline\alpha \varepsilon \mid X \in C',\,\|\nabla f(X)\| \ge \beta)-\delta-\gamma.
\end{eqnarray}
\label{thm:almostconstgrad}
\end{restatable}
\vspace{-.5cm}
\begin{restatable}[Tightness]{rmk}{}
Theorem \ref{thm:almostconstantgrads} is tight, as can be seen by considering the case where $C$ is a half-space for which $f(x) = x^\top w - b$, for some unit-vector $w \in \sphere$, and $b \in \mathbb R$; take $V = \mathbb Rw$. \textbf{N.B.:} $\nabla f(x) \equiv w$, and so Conditions \ref{cond:beta} and \ref{cond:almostconstgrad} hold with $\alpha=\beta=1$, $\theta=\gamma=0$, and $R=\infty$.
\end{restatable}

\subsection{Application to feed-forward ReLU neural networks}
\label{subsec:relu}
Consider a feed-forward neural net with ReLU activation and $M \ge 2$ layers with parameters matrices $W_1 \in \mathbb R^{d_0 \times d_1},W_2 \in \mathbb R^{d_1 \times d_2},\ldots,W_M=a \in \mathbb R^{d_{M-1} \times d_M}$,
where $d_0 = d$ and $d_M := 1$. Each $d_\ell$ is the width of the $\ell$ layer, and the matrices $W_1,\ldots,W_{M-1}$ are the intermediate weights matrices, while $W_M=a$ is the output weights vector. For aninput $x \in \mathbb R^d$, the output of the neural net is
\begin{eqnarray}
\begin{split}
f_{\mathrm{relu}}(x) = z_M &:= a^\top z_{M-1}\in \mathbb R,\text{ with }
z_0 := x,\, z_\ell := \mathrm{relu}(W_\ell^\top z_{\ell-1}) \in \mathbb R^\ell,\,\forall \ell \in [M-1],
\end{split}
\label{eq:relunn}
\end{eqnarray}
and the ReLU activation is applied entry-wise.
The matrices $W_1,\ldots,W_M$
are randomly initialized: 
\begin{eqnarray}
\begin{split}
[W_\ell]_{i,j} &\overset{iid}{\sim} N(0,1/d_{\ell-1}),\text{ for }\ell \in [M],i \in [d_\ell],\, j \in [d_{\ell-1}].
\end{split}
\label{eq:randomweights}
\end{eqnarray}
The output weights vector $a \in \mathbb R^{d_{M-1}}$ can be arbitrary, for example:
(1) random (as in \cite{daniely2020most, bartlett2021}), or
(2) optimized to fit training data, as in the so-called random features (RF) regime \citep{rf,rf2}, with $L_2$-regularization on $a$.

Let $d_{\min} := \min_{0 \le \ell \le M-1}d_\ell$ and $d_{\max} := \max_{0 \le \ell \le M-1} d_\ell$ be respectively, the minimum and maximum width of the layers. As in \cite{bartlett2021}, assume the following condition.
\begin{restatable}[Genuinely wide, finite-width]{cond}{}
\label{cond:bottleneck}
The neural network architecture verifies:
(i) Bounded depth, i.e., $M=\mathcal O(1)$ layers. (ii)
Genuinely wide, i.e., $d_{\min} \gtrsim (\log d_{\max})^{40M}$ and $d_{\min} \to \infty$.
\end{restatable}
We have the following corollary to Theorem \ref{thm:almostconstantgrads}.
\begin{restatable}[Feed-forward ReLU neural networks with random intermediate weights]{cor}{relu}
Consider the decision-region $C = \{x \in \mathbb R^d \mid f_{\mathrm{relu}}(x) \le 0\}$ where $f_{\mathrm{relu}}$ is the $M$-layer feed-forward ReLU neural network defined in \eqref{eq:relunn} with random intermediate weights $W_1,\ldots,W_{M-1}$ sampled according to \eqref{eq:randomweights}. 
Suppose Conditions \ref{cond:bottleneck} is in order. Let $V$ be a possibly random $(\alpha,\delta)$-viable subspace of $\mathbb R^d$, with $\alpha=\Omega(1)$. Then,
for $0 \le \varepsilon \lesssim (\log d_{\max})^{40M}$, it holds w.h.p. over $W_1,\ldots,W_{M-1}$ that 
\begin{eqnarray}
\mathbb E_V [\fr(V;\varepsilon)] \gtrsim (1-\delta)\mathbb P_X(m_{f_{\mathrm{relu}}}(X) \le \varepsilon \mid X \in C').
\end{eqnarray}


In particular, at initialization, for $\varepsilon \gtrsim \mathbb E\|X\|/\sqrt{d}$ we have,
$\mathbb E_V[\fr(V;\varepsilon)] \gtrsim 1-\delta$.
\label{cor:relu}
\end{restatable}
The second part of the result implies that the subspace $V$ contains adversarial perturbations of size $\sqrt{d}$ times smaller than the norm of a typical datapoint. Thus, it is a generalizes \cite{daniely2020most,bartlett2021} to subspaces.

%% file: compact.tex
\section{Universal adversarial perturbations for compact decision-regions}
\label{sec:compactum}
We now consider the case where (i) the positive decision-region $C' := \mathbb R^d \setminus C$ of the classifier \eqref{eq:fC} is an arbitrary \emph{compact} subset of $\mathbb R^d$, and (ii) the distribution of feature vector $X$ conditioned on  $X \in C'$ is the uniform distribution on $C'$.
We show that a single adversarial
direction $v$ is sufficient to switch a nonzero fraction ofinputs from the positive decision-region $C'$ to the negative one $C$.
\begin{wrapfigure}[9]{R}{.45\textwidth}
    \centering
    \vspace{-.3cm}
    \caption{Lens-shaped region $A_d^{\mathrm{lens}}(\varepsilon/2)$.
    In high-dimensions,
    its relative volume of tends to zero (curse of dimensionality)
    }
    \includegraphics[width=.43\linewidth]{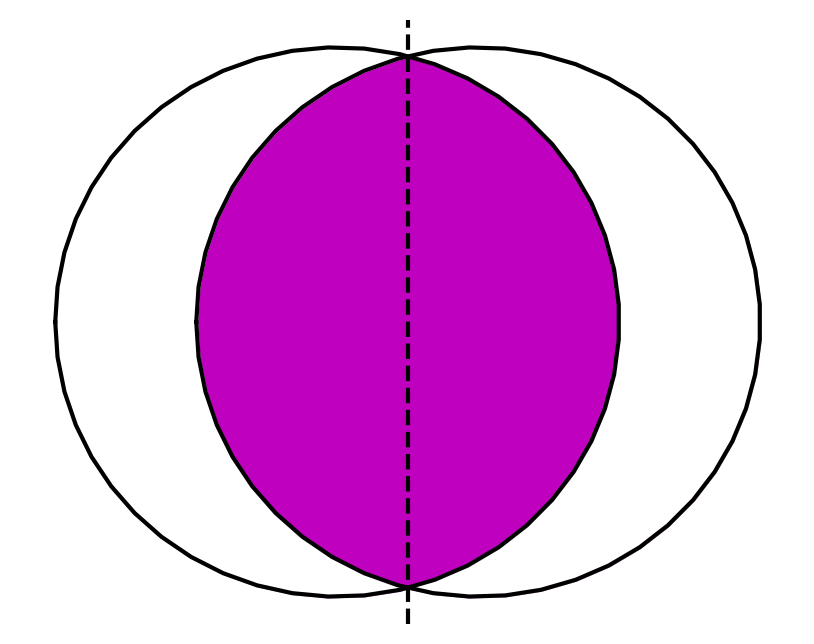}
    \label{fig:lens}
\end{wrapfigure}
\subsection{Warm-up: The case of a ball}
Suppose the positive decision-region  $C' := \mathbb R^d \setminus C$ of the classifier \ref{eq:fC} is the unit-ball $B_d$ and the distribution of the features $X \in \mathbb R^d$ conditioned on $X \in C'$ is uniform on $C'$. Thus, the negative decision-region is $C = \{x \in \mathbb R^d \mid f(x) \le 0\}$, where $f(x) := (1-\|x\|^2)/2$. 
Consider the unit-vector $v=(1,0,\ldots,0) \in \mathbb R^d$, and let
\begin{eqnarray}
A_d^{\rm cap}(r) := \{x \in \mathbb R^d \mid v^\top x \ge
r\}
\end{eqnarray}
be the spherical cap of height
$r \ge 0$, and $A_d^{\rm lens}(r)$ be the corresponding spherical lens, of volume
twice that of $A_d^{\rm cap}(r)$.
By geometry of the situation (refer to Fig. \ref{fig:lens}), the fooling rate
of $v$ is given by
\begin{eqnarray}
\begin{split}
   \fr(v;\varepsilon) \ge \mathbb P_X(X + \varepsilon v \in C \mid X \in C')
    &= 1-\frac{\vol(B_d \cap (\varepsilon v + B_d))}{\vol(B_d)}
    = 1 -
    \frac{\vol(A^{\rm lens}_d(\varepsilon/2))}{\omega_d},
\end{split}
\label{eq:lens}
\end{eqnarray}
where 
$\omega_d=\pi^{d/2}/\Gamma(d/2+1)$ is the volume of the unit-ball $B_d$. Note that $A_d^{\rm lens}(\varepsilon/2)$, has diameter $2\sqrt{1-\varepsilon^2/4}$, and therefore
is contained in a ball of radius $\sqrt{1-\varepsilon^2/4}$.
Thus, we have
\begin{eqnarray}
\vol(A_d^{\rm len}(\varepsilon/2)) \le (1-\varepsilon^2/4)^{d/2}\omega_d \le
e^{-d\varepsilon^2/8}\omega_d.
\end{eqnarray}
Further, if $\sqrt{8/d} \le \varepsilon < 2$, with a bit more work one get the improved upper-bound
\begin{eqnarray}
\vol(A_d^{\rm len}(\varepsilon/2)) \le
(\varepsilon \sqrt{d}) e^{-\varepsilon^2(d-1)/8}\omega_d.
\end{eqnarray}

For example, see ~\cite[page 221]{boucheron2013}.
Combining with \eqref{eq:lens}, we deduce the following result lower-bounding the fooling rate of one-dimensional subspaces.
\begin{restatable}{thm}{}
\label{thm:ballcase}
Suppose the positive decision-region $C' := \mathbb R^d \setminus C$ is the unit-ball $B_d$. Then, given an attack budget $\varepsilon \ge 0$, the fooling rate $\fr(v;\varepsilon)$ of any unit-vector $v \in \sphere$ is lower-bounded like so
\begin{eqnarray}
\fr(v;\varepsilon) \ge \mathbb P(X + \varepsilon v \in C \mid X \in C') \ge g_d(\varepsilon),
\end{eqnarray}
where $g_d:\mathbb R_+ \to [0,1]$ is the function defined by
\begin{eqnarray}
\label{eq:expobound}
\begin{split}
 g_d(\varepsilon)\begin{cases}=1,&\mbox{ if
    }\varepsilon \ge 2,\\
    \ge 1-(\varepsilon\sqrt{d})^{-1}e^{-\varepsilon^2(d-1)/8}, &\mbox{ if }\sqrt{8/d} \le \varepsilon <
    2,\\
    \ge 1-e^{-\varepsilon^2d/8}, &\mbox{ if }0 \le \varepsilon < \sqrt{8/d}.\end{cases}
    \end{split}
\end{eqnarray}
\end{restatable}
Of course, it is reminiscent of the \emph{curse of dimensionality} that, for every fixed attack budget $\varepsilon>0$, the lower-bound $g_d(\varepsilon)$ increases rapidly to $1$ as a function of theinput dimension $d$.


\subsection{The case of general compact bodies}
We extend Theorem \ref{thm:ballcase} to the case where the positive decision-region $C':=\mathbb R^d\setminus C$ is a (nonempty) compact subset of $\mathbb R^d$ and let $R(C')$ be the radius of the ball with the same volume as $C'$. Using an argument based on the \emph{Riesz-Sobolev rearrangement inequality} \citep{riez},
we reduce to the previously discussed ball case and establish the following result, which is one of our main contributions.
\begin{restatable}[Universal adversarial perturbation for compact decision-region]
  {thm}{steiner}
Suppose the positive decision-region $C':=\mathbb R^d \setminus C$ is a compact subset of $\mathbb R^d$ equipped with the uniform measure. Then, for any $\varepsilon \ge 0$, there exists a direction $v \in \sphere$ with fooling rate lower-bounded as \begin{eqnarray}
\fr(v;\varepsilon) \ge g_d(\varepsilon/(2R(C'))),
\end{eqnarray}
where the function $g_d$ is given in \eqref{eq:expobound}.
\label{thm:steiner}
\end{restatable}

%% file: conclusion.tex
\section{Concluding remarks}
We conducted a rigorous analysis of the phenomenon of low-dimensional adversarial perturbations and derived tight lower-bounds for the fooling rate along arbitrary adversarial subspaces based on the geometry of the target decision-region, and the alignment between the subspace and the gradients of the model, i.e., the adversarial viability of the subspace (Definition \ref{df:viableV}). Our work provides rigorous foundations for explaining intriguing empirical observations from the literature on the subject ~\citep{unipert,singularart,yin2019fourier,guo2018low}. For the case of compact decision-regions we have shown the existence of UAPs. We believe our work will further generate fruitful research in this area.




%% file: appendix.tex
\appendix

\begin{center}
    
    \Large \textbf{Appendix}
    
\end{center}


\input{viable}
\input{proofs}
\input{misc_experiments}

%% file: viable.tex
\section{Adversarially viable subspaces}
\label{sec:infoV}
Below, we list a few important examples of adversarial viable subspaces (refer to Definition \ref{df:viableV}).

\subsection{The linear span of the gradient field} The subspace $V_{all}$, spanned by the set of gradients $\{\nabla f(x) \mid x \in C'\}$ is adversarially $(1,0)$-viable, since it induces no distortion at all: it preserves the entire norm of the gradient of $f$ at any point of the positive decision-region $x \in C'$. The same is true in the trivial case when $V=\mathbb R^d$, the entireinput space $\mathbb R^d$, irrespective of $f$. For example, in the case of a linear classifier with $f(x) := x^\top w -b$, $V_{all}$ is simply the one-dimensional subspace spanned by $w$. For a less trivial example, it is known since \cite{montufar2014,Hanin2019,serra2018} that a ReLU neural net $f$ with a total of $N$ neurons in the intermediate layers, partitions theinput space $\mathbb R^d$ into $P=\mathcal O(2^N)$ pieces and $f$ is an affine function on each of the pieces.
Thus, $V_{all}$ is an adversarially $(1,0)$-viable $P$-dimensional subspace. As a side note, it is thus desirable to design neural networks which have a large number of pieces. This requires over-parametrization, and is consistent with recent findings \citep{lor,bubeck2021universal}.

\subsection{Proof of Lemma \ref{lm:randomviable}}
\randomviable*
\begin{proof}
Fix $x \in \mathbb R^d$ and consider the unit-vector $u = \eta(x):=\nabla f(x)/\|\nabla f(x)\|$.
 From the \emph{Johnson-Lindenstrauss Lemma}, we know that for a uniformly-random $k$-dimensional subspace $V$ of $\mathbb R^d$, it holds w.p $1-e^{-dt^2/2}$ that $\|\Pi_V u\| \ge \sqrt{k/d}-t$. We deduce that 
 \begin{eqnarray*}
 \begin{split}
\mathbb E_V \mathbb P_X(\|\Pi_V \eta(X)\| \ge \sqrt{k/d}-t \mid X \in C') &= \mathbb E_X[\mathbb P_V(\|\Pi_V \eta(X)\| \ge \sqrt{k/d}-t) \mid X \in C']\\
&\ge 1 - e^{-dt^2/2},
\end{split}
 \end{eqnarray*}
 which proves the result.
\end{proof}

\subsection{Proof of Theorem \ref{thm:svd} (eigen-subapces)}
\svd*

\begin{proof}
 Let $\Sigma_\eta = USU^\top$ be the SVD of $\Sigma_\eta$, where $S$ is a diagonal matrix containing the nonzero
eigenvalues $\lambda_1 \ge \lambda_2 \ge \ldots \ge \lambda_r  > 0$ of $\Sigma_\eta$, $r \in [d]$ is the rank of $\Sigma_\eta$, and $U$ is a $d \times r$ matrix with
orthonormal columns. Then, the orthogonal projector for the subspace $V := V_{\mathrm{eigen},k}$ is given explicitly by $\Pi_V = U_{\le k}U_{\le k}^\top$, where $U_{\le k}$ is the $d \times \min(k,r)$ orthogonal matrix corresponding
to the first $\min(k,r)$ columns of $U$. Consider the r.v $Z := \|\Pi_V \eta(X)\|$. By a standard formula for the expectation of a quadratic form, one computes
\begin{eqnarray}
\begin{split}
\mathbb E\, [Z^2 \mid X \in C'] &= \mathbb E[\eta(X)^\top\Pi_V\eta(X) \mid X \in C'] = \trace(\Pi_V\Sigma_\eta)= \trace(U_{\le k}U_{\le k}^\top\Sigma_\eta)\\
&=\trace(U_{\le k}^\top \Sigma_\eta U_{\le k}) = \sum_{i=1}^{\min(k,r)}\lambda_i =: s_k.
\end{split}
\end{eqnarray}
On the other hand, conditioned on the event $X \in C'$ we have $0 \le Z \le \|\eta(X)\|$. Thus, for any $\alpha \in (0,\sqrt{s_k})$, we have
\begin{eqnarray}
X \in C' \implies \mathbbm{1}(Z \ge \alpha) \ge (Z^2-\alpha^2)/(1-\alpha^2),\text{ with equality on the event }Z^2 \in \{\alpha^2,1\}.
\label{eq:pinelis}
\end{eqnarray}
The claim then follows upon taking expectations on both sides of the above inequality conditioned on the event $X \in C'$.
\end{proof}



%% file: proofs.tex


\section{Proof of Theorem \ref{thm:smoothboundary} : Lower-bound assuming Lipschitz decision-boundary}
\label{subsec:smoothboundaryproof}
\smoothboundary*

\input{halspace}

\subsection{Proof idea for Theorem \ref{thm:smoothboundary}}
It is folklore in optimization theory that a function $f$ which satisfies Condition \ref{cond:smooth} admits the following first-order approximation
\begin{eqnarray}
-\frac{L}{2}\|x'-x\|^2 \le f(x') - f(x)-\nabla f(x)^\top (x'-x) \le \frac{L}{2}\|x'-x\|^2,\text{ for all }x,x' \in \mathbb R^d.
\label{eq:gradlipineq}
\end{eqnarray}

Starting at a point $x \in C'$, move a distance $\varepsilon$, in the direction $\Pi_V \nabla f(x)$ to arrive at a point $x' = x-\varepsilon \Pi_V \nabla f(x) \in \mathbb R^d$, the above inequality gives the quadratic approximation 
\begin{eqnarray}
f(x') \le f(x)-\varepsilon \|\Pi_V \nabla f(x)\|^2 + \frac{L}{2}\varepsilon^2\|\Pi_V \nabla f(x)\|^2.
\end{eqnarray}

The RHS can be made $\le 0$ by guaranteeing that
\begin{itemize}
\item[(1)] $\|\Pi_V \nabla f(x)\| \ge \alpha\|\nabla f(x)\|$.
\item[(2)] $m_f(x) \le \min(\alpha \varepsilon/2,\alpha^2\|\nabla f(x)\|/(2L))$.
\end{itemize}

\subsection{Additional notations} 
We will need some additional notations. Let
$d(x) \in [0,\infty)$ be the distance of $x$ from $C$ and let $d_V(x) \in [0,\infty]$ be the distance of $x$ from $C$ along the subspace $V$, i.e.,
\begin{eqnarray}
\begin{split}
d(x) &:= \inf_{v \in \mathbb R^d}\|v\|\text{ subject to }x+v \in C,\\
d_V(x) &:= \inf_{v \in V} \|v\|\text{ subject to }x + v \in C,
\end{split}
\label{eq:distances}
\end{eqnarray}
with the convention that $\inf \emptyset = \infty$. By definition of the $(\varepsilon,V)$-expansion $C_V^\varepsilon$ of $C$, we have
\begin{eqnarray}
\label{eq:torch}
C^\varepsilon_V = \{x \in \mathbb R^d \mid d_V(x) \le \varepsilon\}.
\end{eqnarray}
Also, it is clear that $d_V(x) \ge d(x)$, attained when $V = \mathbb R^d$. As will see in the proof of Theorem \ref{thm:smoothboundary} below, turns out that if the gradients $\nabla f(x)$ for $x \in C'$ are well-aligned (but not necessarily perfectly) with the subspace $V$, then there is an upper-bound of the form $d_V(x) \lesssim m_f(x)$, where $m_f(x)$ is the \emph{margin} of $f$ at $x$ defined in \eqref{eq:margin}.


\subsection{Auxiliary lemmas}
\begin{restatable}{lm}{huber}
For any $\rho,r >0$ and $b \in \mathbb R^d$, we have the identity
\begin{align}
\sup_{z \in \rho B_n}b^\top z - \frac{1}{2r}\|z\|^2 &= \begin{cases}
r\|b\|^2/2,&\mbox{ if }\|b\| \le \rho / r,\\
\rho\|b\|-\rho^2/(2r),&\mbox{ otherwise.}
\end{cases}
\end{align}
\label{lm:huber}
\end{restatable}
\begin{proof}
Since the quadratic function $z \mapsto (1/2)\|z\|^2$ is unchanged upon taking the \emph{Fenchel-Legendre transform}, we have
\begin{eqnarray*}
\begin{split}
\sup_{z \in \rho B_d}b^\top z - \frac{1}{2r}\|z\|^2 &= \sup_{\|z\| \le \rho}b^\top z - \frac{1}{r} \left(\sup_{u \in \mathbb R^d}z^\top u - \frac{1}{2}\|u\|^2\right)\\
&\stackrel{(*)}{=} \inf_{u \in \mathbb R^d}\left(\frac{1}{2r}\|u\|^2+\sup_{\|z\| \le \rho} z^\top (b-u/r)\right)\\
&= \inf_{u \in \mathbb R^d}\left(\frac{1}{2r}\|u\|^2 + \rho \|b-u/r\|\right)\\
&= \inf_{v \in \mathbb R^d}\left(\frac{r}{2}\|v-b\|^2+\rho\|v\|\right),\text{ by change of variable }v := b-u/r\\
&= \rho\inf_{v \in \mathbb R^d}\left(\frac{1}{2\rho/r}\|v-b\|^2+\|v\|\right),\text{ by factoring out }\rho\\
&\stackrel{(**)}{=} \rho \begin{cases}
\|b\|^2 / (2\rho/r),&\mbox{ if }\|b\| \le \rho/r,\\
\|b\|-\rho/(2r),&\mbox{ else}
\end{cases}\\
&=\begin{cases}
r\|b\|^2/2,&\mbox{ if }\|b\| \le \rho/r,\\
\rho\|b\|-\rho^2/(2r),&\mbox{ else,}
\end{cases}
\end{split}
\end{eqnarray*}
where $(*)$ uses \emph{Sion's Minimax Theorem}, and in $(**)$ we have recognized a rescaled \emph{Moreau envelope} of the Euclidean norm, which is the Huber function evaluated at $\|b\|$.
\end{proof}

Finally, we will need the following lemma.

\begin{restatable}{lm}{}
Suppose $R_1,R_2,R_3$ are random variables and $\phi:\mathbb R \to \mathbb [-\infty,\infty]$ is a possibly random nondecreasing function. If $\mathbb P(R_2 \ge R_3) \ge 1-\delta$
\begin{eqnarray}
\mathbb P(R_1 \le \phi(R_2)) \ge \mathbb P(R_1 \ge \phi(R_3))-\delta.
\end{eqnarray}
\label{lm:proba}
\end{restatable}
\begin{proof}
Indeed, consider the events $E_1:=\{R_1 \le \phi(R_3)\}$, $E_2:=\{R_3 \le R_2\}$, $E_3:=E_1 \cap E_2$ and $E_4 := \{R_1 \le \phi(R_2)\}$. It is clear that $E_3 \subseteq E_4$. One then easily computes
\begin{eqnarray*}
\begin{split}
\bP(R_1 \le \phi(R_2)) = \bP(E_4) &\ge \bP(E_3) = \bP(E_1 \cap E_2)\\
&= \bP(E_1) + \bP(E_2) - \bP(E_1 \cup E_2)
\\
&\ge \bP(E_1) + \bP(E_2) - 1\\
&\ge \bP(E_1) - \delta\\
&= \bP(R_1 \le \phi(R_3)) - \delta,
\end{split}
\end{eqnarray*}
as claimed.
\end{proof}

\subsection{Proof of Theorem \ref{thm:smoothboundary}}
We are now ready to prove Theorem \ref{thm:smoothboundary}.
\begin{proof}[Proof of Theorem \ref{thm:smoothboundary}]
Let $x \in C':=\mathbb R^d\setminus C$ 
and set $v(x) := \Pi_V \nabla f(x)/\|\Pi_V \nabla f(x)\| \in \sphere \cap V$. Define $p_V(x) := \|\Pi_V \nabla f(x)\|$, the $L_2$-norm of the orthogonal projection of the gradient vector $\nabla f(x)$ onto the subspace $V$. It is clear that $\nabla f(x)^\top v(x) = \|\Pi_V \nabla f(x)\| = p_V(x)$. Let $d_V(x) \in (0,\infty]$ be the distance of $x$ from $C$ along the subspace $V$ (see \eqref{eq:distances}). By definition, $d_V(x)$ is no larger than the distance between $x$ and the point where the line $x + \mathbb R v(x) := \{x + sv(x) \mid s \in \mathbb R\}$ first meets $C$ (if it meets it at all!). Thus, with the convention $\inf \emptyset = \infty$, we have
\begin{eqnarray}
\begin{split}
    d_V(x) &\le \inf_{s \in \mathbb R}|s| \text{ subject to }x + sv(x) \in C\\
    &= \inf_{s \in \mathbb R} |s|\text{ subject to }f(x + sv(x)) \le 0\\
    &\le \inf_{s \in \mathbb R} |s| \text{ subject to }f(x) + s\nabla f(x)^\top v(x) + Ls^2/2 \le 0\\
    &= \inf_{s \in \mathbb R}|s|\text{ subject to }f(x) + p_V(x)s + Ls^2/2 \le 0,
\end{split}
\end{eqnarray}
where we have invoked the RHS of \eqref{eq:gradlipineq} with $x'=x+sv(x)$ to arrive at the third line.
\begin{eqnarray}
\begin{split}
f(x) \ge \sup_{|s| < d_V(x)}-p_V(x) s - Ls^2/2  &= \begin{cases}
p_V(x)^2/(2L),&\mbox{ if }p_V(x) \le Ld_V(x),\\
p_V(x)d_V(x)-Ld_V(x)^2/2,&\mbox{ otherwise,}
\end{cases}
\end{split}
\label{eq:branches}
\end{eqnarray}
where the second step is an application of Lemma \ref{lm:huber} with $n=1$, $b=-p_V(x)$, $r=1/L$ and $\rho = d_V(x)$.
\begin{figure}[!h]
    \centering
    \includegraphics[width=.8\linewidth]{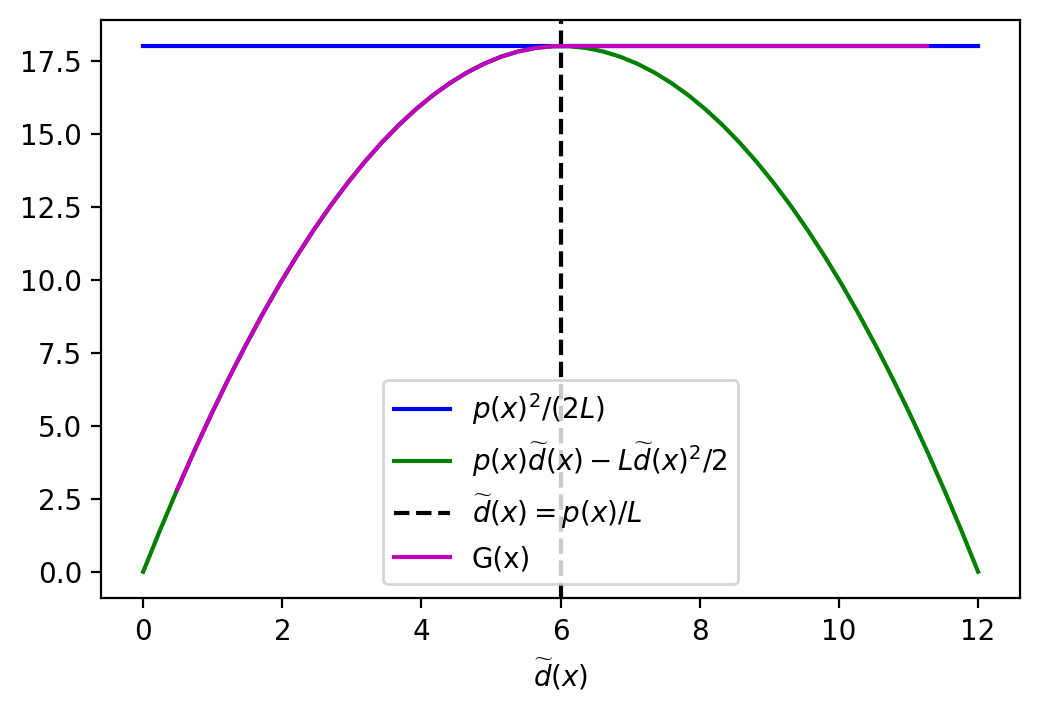}
    \caption{Graphical illustration of the RHS of \eqref{eq:branches}, denote here as $G(x)$. In this illustration, $p(x)=p_V(x)$ and $L$ are fixed to $5$ and $1$ respectively. Here, $\widetilde d(x)$ is shorthand for $d_V(x)$, the distance of $x$ from $C$ along the subspace $V$.}
    \label{fig:picture}
\end{figure}
Now, if $f(x) < p_V(x)^2/(2L)$, we deduce from \eqref{eq:branches} that $d_V(x) < p_V(x)/L$ and $f(x) \ge p_V(x)d_V(x)-Ld_V(x)^2/2$ (see Figure \ref{fig:picture} for geometric intuition), and so
\begin{eqnarray}
\begin{split}
d_V(x) &\le p_V(x)/L - \sqrt{(p_V(x)/L)^2 - 2f(x)/L} = \frac{2f(x)}{p_V(x) + \sqrt{p_V(x)^2-2f(x)L}}\\
&\le \frac{2f(x)}{p_V(x)} = 2\alpha_V(x)m_f(x),
\end{split}
\label{eq:quad}
\end{eqnarray}
where $\alpha_V(x) = p_V(x)/\|\nabla f(x)\| = \|\Pi_V \nabla f(x)\|/\|\nabla f(x)\| = \|\Pi_V \eta(x)\|$.

Because $C^\varepsilon_V = \{x \in \mathbb R^d \mid d_V(x) \le \varepsilon\}$, we deduce that
\begin{eqnarray}
\begin{split}
C^\varepsilon_V\setminus C \supseteq \left\{x \in C' \mid m_f(x) \le \min(\frac{\alpha_V(x)\varepsilon}{2},\frac{\alpha_V(x)^2\|\nabla f(x)\|}{2L})\right\}.
\end{split}
\end{eqnarray}

Now, define $s_V(x):=\alpha_V(x)^2\|\nabla f(x)\|/(2L)$ and $s(x):=\alpha^2\|\nabla f(x)\|/(2L)$. Since the subspace $V$ is an adversarial $(\alpha,\delta)$-viable by hypothesis, we have, it follows from Definition \ref{df:viableV} that
\begin{eqnarray}
\bP_{X,V}(s_V(X) \ge s(X) \mid X \in C') \ge \bP_{X,V}(\|\Pi_V \eta(X)\| \ge \alpha \mid X \in C') \ge 1-\delta.
\label{eq:sVbound}
\end{eqnarray}
The \emph{Fubini-Tonelli Theorem} then gives,
\begin{eqnarray*}
\begin{split}
\fr(V;\varepsilon) &:= \mathbb E_V \mathbb P_X(X \in C^\varepsilon_V \mid X \in C') = \mathbb E_X \mathbb P_V(X \in C^\varepsilon_V \mid X \in C')\\
&\ge \mathbb E_X\bP_V(m_f(X) \le \min(\alpha\varepsilon/2, s_V(X)) \mid X \in C')\\
&\ge \mathbb E_X\bP_V(m_f(X) \le \min(\alpha\varepsilon/2, s(X)) \mid X \in C')-\delta,
\end{split}
\end{eqnarray*}
where the last step is thanks to Lemma \ref{lm:proba} with $R_1=m_f(X)$, $R_2 = s_V(X)$, $R_3 = s(X)$, and $\phi=Id$, and recalling \eqref{eq:sVbound}.
This proves the first part of the theorem.

For the second part, Condition \ref{cond:beta} is in order and so we have $\mathbb P(\|\nabla f(X)\| \ge \beta \mid X \in C') \ge 1 - \gamma$.

\end{proof}

\input{smooth_examples}

\subsection{Proof of Theorem \ref{thm:upperbound}: An upper-bound (tightness of Theorem \ref{thm:smoothboundary})}
\label{subsec:upperboundproof}
As usual, we restate the result for easier reference.
\upperbound*
\begin{proof}
By definition of $d_V(x)$ (see \eqref{eq:distances}), it is clear that $x - d_V(x) v \in C$, where $v = \Pi_V \nabla f(x) / \|\Pi_V \nabla f(x)\|$. Observe that $\nabla f(x)^\top v = \|\Pi_V \nabla f(x)\|$. Now, thanks to the convexity of $f$, we have
\begin{eqnarray}
\begin{split}
    x - d_V(x) v \in C &\implies f(x - d_V(x) v) \le 0 \implies f(x) - d_V(x)\nabla f(x)^\top v \le 0 \\
    &\implies m_f(x) \le \frac{d_V(x)\nabla f(x)^\top v}{\|\nabla f(x)\|} \le \frac{d_V(x)\|\Pi_V \nabla f(x)\|}{\|\nabla f(x)\|} \le \widetilde\alpha d_V(x).
\end{split}
\end{eqnarray}
Thus, $\{x \in C' \mid m_f(x) \le \widetilde\alpha \varepsilon\} \supseteq \{x \in C' \mid d_V(x) \le \varepsilon\} =: C^\varepsilon_V\setminus C$, and the result follows.
\end{proof}

\subsection{Proof of Theorem \ref{thm:almostconstantgrads}: Locally affine decision-region}
\label{subsec:almostconstantgradsproof}
\almostconstantgrads*
We will need the following auxiliary lemma.
\begin{restatable}{lm}{softthres}
\label{lm:softthres}
For any $r,\rho > 0$ and $b \in \mathbb R^d$, we have the identity
\begin{eqnarray}
\sup_{z \in \rho B_n}b^\top z - \frac{1}{r}\|z\| = \rho(\|b\|-1/r)_+.
\end{eqnarray}
\end{restatable}

\begin{proof}
By direct computation, we have
\begin{eqnarray*}
\begin{split}
\sup_{\|z\| \le \rho}b^\top z - \frac{1}{r}\|z\| &= \sup_{\|z\| \le \rho}b^\top z-\sup_{\|u\| \le 1}z^\top u/r\\
&= \inf_{\|u\| \le 1}\sup_{\|z\| \le \rho}z^\top(b-u/r)\\
& = \rho \inf_{\|u\| \le 1}\|b-u/r\|\\
&= \rho (\|b\|-1/r)_+,
\end{split}
\end{eqnarray*}
we in the last step, we have recognized the well-known Euclidean \emph{soft-thresholding} operator.
\end{proof}

\begin{proof}[Proof of Theorem \ref{thm:almostconstantgrads}]
Under Condition \ref{cond:almostconstgrad}, it is easy to establish the classical inequality
\begin{eqnarray}
-\theta\|x'-x\| \le f(x') - f(x) - \nabla f(x)^\top (x'-x) \le \theta\|x'-x\|,\text{ for all }\|x'-x\| \le R.
\label{eq:theta}
\end{eqnarray}
Now, let $x \in C' := \mathbb R^d \setminus C$ and let $d_V(x)$ be the distance of $x$ from $V$ along the subspace $V$. Let $v(x)$, $p_V(x)$, $\alpha_V(x)$, $s_V(x)$, and $s(x)$ be as defined in the proof of Theorem \ref{thm:smoothboundary}. By an argument analogous to the beginning of the proof of Theorem \ref{thm:smoothboundary} but with \eqref{eq:theta} used in place of \eqref{eq:gradlipineq} and the restriction that $|s| \le R$ so that \eqref{eq:theta} is valid for every $x'$ on the line $x+\mathbb Rv(x)$, it is straightforward to establish that
\begin{eqnarray}
\begin{split}
d_V(x) &\le \inf_{s \in \mathbb R}|s| \text{ subject to }x+sv(x) \in C,\,|s| \le R\\
&\le \inf_{s \in \mathbb R}|s| \text{ subject to }f(x) + p_V(x)s + \theta|s|\le 0,\,|s| \le R\\
&\le \inf_{s \in \mathbb R}|s| \text{ subject to }f(x) + p_V(x)s + \theta|s|\le 0,\,|s| \le R.
\end{split}
\end{eqnarray}
We deduce that
\begin{eqnarray}
\begin{split}
f(x) &\ge \sup_{|s| < \min(d_V(x),R)}-p_V(x)s - \theta|s| = \min(d_V(x),R)\cdot (p_V(x)-\theta)_+,
\end{split}
\label{eq:xyz}
\end{eqnarray}
where the equality is thanks to Lemma \ref{lm:softthres} applied with $n=1$, $b=-p_V(x)$, $r=1/\theta$, and $\rho = \min(d_V(x),R)$.
Thus, we deduce from \eqref{eq:xyz} that
\begin{eqnarray}
\begin{split}
\min(d_V(x),R) \le \frac{f(x)}{(p_V(x)-\theta)_+} =  \frac{1}{\overline \alpha_V(x)}m_f(x),
\end{split}
\end{eqnarray}
with $\overline\alpha_V(x) := \|\nabla f(x)\|/(\alpha_V(x)\|\nabla f(x)\|-\theta)_+$.
Thus, if $m_f(x) \le \overline \alpha_V(x)\varepsilon$ and $\varepsilon < R$, then $d_V(x) \le \varepsilon$.
The rest of the proof is analogous to the case of Theorem \ref{thm:smoothboundary}, as is thus omitted.
\end{proof}

\subsection{Proof of Corollary \ref{cor:relu}: ReLU neural networks}
\relu*
The result is obtained as a consequence of Theorem \ref{thm:almostconstantgrads}, by combining Lemma 2.2 and Lemma 2.8 of \cite{bartlett2021} and the following lemma
 \begin{restatable}{lm}{}
 \label{lm:bartlett}
 Suppose Condition \ref{cond:bottleneck} is in order. Then, w.h.p. over the random intermediate weights $W_1,\ldots,W_{M-1}$, the ReLU neural network $f_{\mathrm{relu}}$ satisfies Conditions \ref{cond:beta} and \ref{cond:almostconstgrad} with
\begin{align}
R=R_{\mathrm{relu}} &\gtrsim (\log d_{\max})^{40 M},\\
\theta = \theta_{\mathrm{relu}} &\lesssim \|a\|/(\log d_{\max})^M
\label{eq:lipchainrule},\\
\beta=\beta_{\mathrm{relu}} &= \inf_{x \in C'}\|\nabla f_{\mathrm{relu}}(x)\| \ge  \inf_{x \in \mathbb R^d}\|\nabla f_{\mathrm{relu}}(x)\| \gtrsim \|a\|\label{eq:daniely}.
\end{align}
 \end{restatable}

\section{Proof of Theorem \ref{thm:steiner}: Compact decision-regions}
\steiner*
We will prove Theorem \ref{thm:steiner} by reducing to the case of balls, and then invoking Theorem \ref{thm:ballcase}.



\begin{restatable}[Iso-volumentric radius]{df}{}
The iso-volumetric radius
of a measurable subset $K$ of $\mathbb R^d$, denoted $R(K)$, is the unique $r
\in [0,\infty]$ such that $K$ has the same volume as the ball $B_d(r)$.
i.e.,
\begin{eqnarray*}
R(K) := (\frac{\vol(K)}{\omega_d})^{1/d} \approx \sqrt{\frac{2d}{\pi e}}\vol(K)^{1/d},
\end{eqnarray*}
where $\omega_d = \pi^{d/2}/\Gamma(d/2+1)$ is the volume of the unit-ball $B_d$.
   \end{restatable}
  For example, the hypercube $[a,b]^d$ has iso-volumetric radius $R([a,b]^d) = (b-a)\sqrt{2d/(\pi e)}$, any unbounded $K$ has $R(K) = \infty$, and of course it holds that $R(B_d(r)) = r$.

\begin{proof}[Proof of Theorem \ref{thm:steiner}]
  For a uniformly-random  $v \in B_d$, one computes the average fooling rate of $v$ as
  $\E_v [\fr(v;\varepsilon)] \ge \E_v\mathbb P_X(X + \varepsilon v \in C \mid X \in C') = 1-\tau_{C'}(\varepsilon)$, where
  \begin{eqnarray}
  \label{eq:tau}
  \tau_{K}(\varepsilon) := \E_v \dfrac{\vol(K \cap (\varepsilon v +
      K))}{\vol(K)} \in [0,1].
\end{eqnarray}
It is clear that $\tau_K(\varepsilon) = \tau_{K/\varepsilon}(\varepsilon)$ for any compact subset $K$ of $\mathbb R^d$ and for any $t \ge 0$.
  The result then follows from Theorem \ref{thm:ballcase}. Invoking Lemma \ref{lm:fedja} below then gives
  $$
  \tau_{C'}(\varepsilon) = \tau_{C'/R(C')}(\varepsilon/R(C')) \le \tau_{B_d}(\varepsilon/R(C')),
  $$
  and the result follows directly from \eqref{thm:ballcase}.
 \end{proof}
 
 \begin{restatable}[$\tau$ is maximized by balls]{lm}{fedja}
 \label{lm:fedja}
 Let $\tau$ be the function defined in \eqref{eq:tau}. Then, for every $\varepsilon \ge 0$ and every compact subset $K$ of $\mathbb R^d$, it holds that $\tau_K(\varepsilon) \le \tau_{B(R(K))}(\varepsilon)$.
 \end{restatable}

The lemma is a special case of the following general result.

\begin{restatable}[A rearrangement inequality]{lm}{}
\label{lm:riez}
For nonempty compact subsets $K_1$, $K_2$, $K_3$ of $\mathbb R^d$, and define  $T(K_1,K_2,K_3)$ by
\begin{eqnarray}
T(K_1,K_2,K_3) := \int_{K_1}\vol(K_2 \cap (u+K_3))\,\mathrm{d}x.
\end{eqnarray}
Then, the following inequality holds
\begin{eqnarray}
T(K_1,K_2,K_3) \le T(B_d(R(K_1)), B_d(R(K_2)), B_d(R(K_3))),
\label{eq:riez}
\end{eqnarray}
where, as usual, $B_d(R(K))$ is the centered ball of radius $R(K)$, which has the same volume as $K$.
\end{restatable}
The proof of Lemma \ref{lm:fedja} is obtained from \eqref{eq:riez} by taking $K_1=B_d(\varepsilon)$, the ball of radius $\varepsilon$; $K_3=K_2$; and then normalizing by the volume of the unit-ball $B_d$, namely $\omega_d$.

Lemma \ref{lm:riez} itself is a consequence of the celebrated \emph{Riesz-Sobolev rearrangement inequality}, which we state below for completeness.
\begin{restatable}[The Riesz-Sobolev rearrangement inequality \citep{riez}]{prop}{}
\label{prop:riez}
Let $g_1$, $g_2$, and $g_3$ be nonnegative real-valued functions on $\mathbb R^d$ which vanish at infinity, i.e., $\limsup_{|z| \to \infty}g_i(z) = 0$ for $i=1,2,3$. Then, the following inequality holds
\begin{eqnarray}
\int_{\mathbb R^d}\int_{\mathbb R^d}g_1(x)g_3(x-y)g_2(y)\,\mathrm dx\mathrm dy \le \int_{\mathbb R^d}\int_{\mathbb R^d}g_1^\star(x)g_3^\star(x-y)g_2^\star(y)\,\mathrm dx\mathrm dy,
\end{eqnarray}
where $g^\star$ is the symmetric decreasing rearrangement of $g$, i.e., the unique nonnegative real-valued function on $\mathbb R^d$ such that for every $t\ge 0$, the subset $(g^\star)^{-1}([t,\infty)) := \{x \in \mathbb R^d \mid g^\star(x) \ge t\}$ is a centered ball of the same volume as $g^{-1}([t,\infty))$.
\end{restatable}

\begin{proof}[Proof of Lemma \ref{lm:riez}]
Let $1_{K}$ denote indicator function of a compact set $K$. Compactness implies that $1_{K}$ vanishes at infinity. Notice that we can rewrite  $T(K_1,K_2,K_3) = \widetilde T(1_{K_1},1_{K_2},1_{K_2})$, where
\begin{align}
\widetilde T(g_1,g_2,g_2) = \int_{\mathbb R^d}\int_{\mathbb R^d} g_1(x)g_3(x-y)g_2(y)\,\mathrm dx\mathrm dy.
\end{align}
Now, by Proposition \ref{prop:riez} above, we know that $\widetilde T(g_1,g_2,g_3) \le \widetilde T(g_1^\star,g_2^\star,g_3^\star)$, where $g^\star$ is the symmetric decreasing rearrangement of the function $g$. It then suffices to observe that by definition, $(1_{K})^{\star} = 1_{B_d(R(K))}$. This completes the proof of the lemma.
\end{proof}

\section{Details of experimental setup}
\subsection{Empirical estimation of gradient eigen-subspaces}
Let $x_1,\ldots,x_n \in \mathbb R^d$ be iid samples from $\mathbb P_{X \mid X \in C'}$, the distribution of the data conditioned on the positive decision-region of the classifier, and let $J$ be the $n \times d$ matrix with $i$th row given by $\eta(x_i) := \nabla f(x_i)/\|\nabla f(x_i)\| \in \sphere$. \cite{unipert,singularart} have provided strong empirical evidence that the subspace spanned by the first top eigenvectors of the matrix of $\widehat{\Sigma}_\eta :=J^\top J/n$ contains successful adversarial perturbations. In fact, the one-dimensional subspace spanned by the top eigenvector of $\widehat{\Sigma}_\eta$ was shown in \cite{singularart} to achieve state-of-the-art performance, on a variety of models and datasets. In the following Theorem, we provide a rigorous explanation for the success of these SVD-based heuristics used in \cite{unipert,singularart} to compute UAPs.

\begin{restatable}{rmk}{}
We ignore issues concerning the consistency of approximating the principal eigenvector $\Sigma_\eta$ with that of $\widehat{\Sigma}_\eta$, used in practice \citep{unipert,singularart}.
\end{restatable}

%% file: halspace.tex

\subsection{Warm-up: Half-space}
\label{subsec:warmup}
We start with the simple case of a linear binary classifier on $\bR^d$, for which the negative decision-region (and therefore the positive decision-region too) is a half-space
$C=H_{w,b}$, given by
\begin{eqnarray}
\label{eq:halfspace}
H_{w,b} := \{x \in \mathbb R^d \mid  x^\top w -b \le 0\},
\end{eqnarray}
on with unit-normal vector $w \in \mathbb R^d$ and bias parameter $b \in \mathbb R$. This corresponds to taking $f(x):=x^\top w-b$ in \eqref{eq:fC}.
The following result
was established in Lemma 2.2 of \cite{guo2020threats}.
\begin{restatable}[\cite{guo2020threats}]{prop}{halfspace}
Consider the scenario where
$C$
is the half-space $H_{w,b}$ defined in \eqref{eq:halfspace}. For any subspace $V$ of $\mathbb R^d$ and  $\varepsilon \ge 0$, it holds that
\begin{eqnarray}
\label{eq:convset}
\fr(V;\varepsilon) \ge \bP_X(X^\top w + b \le \|\Pi_V w\|\varepsilon \mid X \in C').
\end{eqnarray}
In particular, if $V$ is a uniformly random $k$-dimensional subspace of $\mathbb R^d$, then for any $t \in (0,\sqrt{k/d})$ it holds w.p $1-2e^{-t^2d/2}$ over $V$ that
$\fr(V;\varepsilon) \ge \bP_X(X^\top w + b \le (\sqrt{k/d}-t)\varepsilon \mid X \in C')$.
\label{prop:halfspace}
\end{restatable}

Observe that, since the margin for the linear classifier is $m_f(x) := \max(f(x),0)/\|\nabla f(x)\| =  (x^\top w + b)_+$, the lower-bound \eqref{eq:convset} can be written as $\mathbb \fr(V;\varepsilon) \ge \mathbb P(m_f(x) \le \alpha\varepsilon \mid X \in C')$. We will emulate this template lower-bound in the next subsection for non-linear classifiers.

\begin{proof}[Proof of Proposition \ref{prop:halfspace}]
We provide a simplified self-contained proof for convenience. Indeed, one computes
\begin{eqnarray*}
\begin{split}
\fr(V;\varepsilon) 
:= \bP_X(X \in C^\varepsilon_V \mid X \in C') &\ge \sup_{v \in V}\bP_X(X \in C_v^\varepsilon \mid X \in C')\\
&= \sup_{v \in V \cap \sphere}\mathbb P_X(X^\top w + \varepsilon v^\top w + b \le 0 \mid X \in C')\\
&= \mathbb P_X(X^\top w + b \le \varepsilon\|\Pi_V w\| \mid X \in C'),
\end{split}
\end{eqnarray*}
which proves the first part of the claim. The second part follows from the first part combined with the \emph{Johnson-Lindenstrauss (JL) Lemma}, whereby $\|\Pi_V w\| \ge \sqrt{k/d}-t$ w.p. $1-2e^{-t^2d/2}$.
\end{proof}

%% file: smooth_examples.tex
\subsection{Application}
\label{subsec:lipschitzexamples}
We provide a non-exhaustive list of examples to illustrate the power of Theorem \ref{thm:smoothboundary} and corollaries.
\paragraph{Half-space.}
This corresponds to taking $f(x) := x^\top w - b$, for some unit-vector $w \in \sphere$ and scalar $b \in \mathbb R$. One computes, $\nabla f(x) = w$, and $\nabla^2 f(x) = 0 \in \mathbb R^{d \times d}$ for any $x \in \mathbb R^d$, and so
    \begin{eqnarray}
    \begin{split}
L &= \sup_{x \in C'}\|\nabla^2 f(x)\|_{op} = 0,\,
\|\nabla f(x)\| = \|w\| = 1\text{ for all }x \in \mathbb R^d,\\
m_f(x) &:= \max(f(x),0)/\|\nabla f(x)\| = (x^\top w - b)_+\text{ for all }x \in \mathbb R^d.
\end{split}
\end{eqnarray}
    Theorem \ref{thm:smoothboundary} then recovers the concrete bounds previously established in Section \ref{subsec:warmup} for half-spaces.
    
\paragraph{Hyper-ellipsoid.} This corresponds to taking $f(x) := (x^\top B x - r^2)/2$, where $B$ is a $d \times d$ positive semi-definite matrix and $r > 0$ is a scalar. One computes $\nabla f(x) = Bx$, $\nabla^2 f(x) = B$, and so
    \begin{align}
    L &= \sup_{x \in C'}\|\nabla^2 f(x)\|_{op} = \|B\|_{op},\\
    \|\nabla f(x)\| &= \|Bx\|\text{ for all }x \in \mathbb R^d,,
    \beta =\inf_{x \in C'}\|\nabla f(x)\| = \inf_{x^\top B x > r^2}\|Bx\| = s_{\min}(B)^{1/2}r,\\
    m_f(x) &= \max(f(x),0)/\|\nabla f(x)\| = (x^\top B x - r^2)_+/(2\|Bx\|)_+, \text{ for all }x \in \mathbb R^d,
    \end{align}
    where $s_{\min}(B)$ is the smallest singular / eigenvalue of $B$.
In particular, if $B=I_d$ in the previous example, so that $C=rB_d \subseteq \mathbb R^d$ is the (origin-centered) closed ball of radius $r>0$, then we deduce $L=1$, $\beta=r$, Moreover, for any $x \in C'$, then the distance of $x$ from $C$ i $d(x) = \|x\|-r$ and we have
\begin{eqnarray}
m_f(x) = (\|x\|^2 - r^2)/\|x\| = (\|x\|-r)(1+r/\|x\|) \in (d(x),2d(x)),
\end{eqnarray}
for all $x \in C'$. Applying Theorem \ref{thm:smoothboundary} then recovers the concrete bounds established in \cite[Lemma 2.3]{guo2020threats} as a special case.
